\documentclass[12pt]{article}
\usepackage[english]{babel}
\usepackage{url}
\usepackage[utf8x]{inputenc}
\usepackage{amsmath}
\usepackage{graphicx}
\graphicspath{{images/}}
\usepackage{parskip}
\usepackage{fancyhdr}
\usepackage{vmargin}
\usepackage{hyperref}
\usepackage{amssymb}
\usepackage{enumitem}
\usepackage{amsthm}
\usepackage{tcolorbox}
\usepackage{amsmath}
\usepackage{mathtools}
\usepackage{graphicx}
\newtheorem{theorem}{Theorem}[section]

\DeclareMathOperator{\argmin}{argmin}
\usepackage[vlined]{algorithm2e}
\usepackage[numbers]{natbib}

\hypersetup{
    colorlinks=true,
    linkcolor=blue,
    filecolor=magenta,      
    urlcolor=cyan,
}
\setmarginsrb{3 cm}{2.5 cm}{3 cm}{2.5 cm}{1 cm}{1.5 cm}{1 cm}{1.5 cm}

\title{Bayes-optimal Hierarchical Classification over Asymmetric Tree-Distance Loss}								
\author{Dheeraj Mekala\\ \textit{CSE Dept, IIT Kanpur} \\Vivek Gupta \\\textit{Microsoft Research, Bangalore}}								
\date{September 23, 2017}											

\makeatletter
\let\thetitle\@title
\let\theauthor\@author
\let\thedate\@date
\makeatother

\begin{document}


\begin{titlepage}
	\centering
    \vspace*{-2 cm}
    \textsc{\LARGE Indian Institute of Technology\\[0.5cm] Kanpur}\\[2.0 cm]	
    \includegraphics[scale = 0.75]{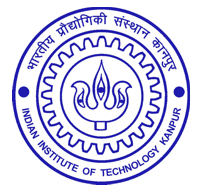}\\[1.0 cm]	
	\textsc{\Large CS396A}\\[0.5 cm]				
	\textsc{\large Undergraduate Project Report}\\[0.5 cm]				
	\rule{\linewidth}{0.2 mm} \\[0.4 cm]
	{ \huge \bfseries \thetitle}\\
	\rule{\linewidth}{0.2 mm} \\[1.5 cm]
	
	\begin{minipage}{0.5\textwidth}
		\begin{flushleft} \large
			\emph{Author:}\\
			\theauthor
			\end{flushleft}
			\end{minipage}~
			\begin{minipage}{0.4\textwidth}
			\begin{flushright} \large
			\vspace{-0.5 cm}
            \emph{Mentor:} \\
			Prof. Purushottam Kar\\ \textit{CSE Dept, IIT Kanpur} \\
			Prof. Harish Karnick\\ \textit{CSE Dept, IIT Kanpur}
		\end{flushright}
	\end{minipage}\\[2 cm]
	
 
	\vfill
	
\end{titlepage}


\tableofcontents
\pagebreak

\section{Abstract}
Hierarchical classification is supervised multi-class classification problem over the set of class labels organized according to a hierarchy. In this
project, we study the work by \citet{ramaswamy2015convex} on hierarchical classification over symmetric tree distance loss. We extend the consistency of hierarchical classification algorithm over asymmetric tree distance loss. We design a $\mathcal{O}(nk\log{}n)$ algorithm to find bayes optimal classification for a k-ary tree as hierarchy. We show that under reasonable assumptions over asymmetric loss function, the
Bayes optimal classification over this asymmetric loss can be found in $\mathcal{O}(k\log{}n)$. We exploit this insight and attempt to extend the Ova-Cascade algorithm \citet{ramaswamy2015convex} for hierarchical classification over asymmetric loss.

\section{Introduction}
Hierarchical Classification is a system of grouping objects according to a hierarchy. Class labels are organized into a pre-defined hierarchy in many practical applications of hierarchical classification. For example, products in
e-commerce industry are generally organized into multilevel hierarchical categories. A general hierarchical classification poses us following challenges: a) Many class labels have data that is extremely sparse. Classifier might get biased to the class label which has larger data, b) Hierarchy forces some constraints on activation of labels. If a node is a true label for a data point, then its parent should also be a possible label, thus, parent node has to be necessarily active, c) The prediction should be fast enough for practical use.\citet{gupta2016product}. Our work settings are similar to that of \citet{ramaswamy2015convex} i.e. class labels are nodes in a tree. We use tree-distance loss \citet{sun2001hierarchical} as our evaluation metric. The main contributions of this project are:

\begin{itemize}
    \item We study Bayes optimal classification over symmetric tree distance loss by \citet{ramaswamy2015convex} and we prove that it is not only sufficient but also necessary for Bayes optimal classification over symmetric tree distance loss.
    \item We propose $\mathcal{O}(nk\log{}n)$ algorithm to find Bayes optimal classification over symmetric/asymmetric loss for a k-ary tree as hierarchy.
    \item Under reasonable assumptions on asymmetric loss, we propose $\mathcal{O}(k\log{}n)$ algorithm to find Bayes optimal classification over asymmetric loss for a k-ary tree as hierarchy and also prove its sufficiency and necessity.
\end{itemize}

\section{Conventions and Notations}
We use the same conventions and notations as that of \citet{ramaswamy2015convex}.
Let the instance space be $\chi$ and let $Y = [n] = {1,..., n}$ be set of class labels. Let $H = ([n], E, W)$  be a tree over the class labels, with edge set $E$, and finite, positive edge weights given by $W$.
$\Delta_n$ denotes the probability simplex in $\mathbb{R}^n: \Delta_n = \{p \in \mathbb{R}^n_+ : \sum_{i} p_i = 1 \}$.\\ \\
For the tree $H = ([n], E, W)$ with root $r$, we define following:\\

$D(y)$ = Set of descendants of $y$ including $y$

$P(y)$ = Parent of $y$

$C(y)$ = Set of Children of $y$

$U(y)$ = Set of ancestors of $y$, not including $y$

$S_y(p) = \sum\limits_{i \in D(y)} p_i$

$l^H(y, y')$ = Symmetric loss i.e. Tree distance loss where $y$ is the true label and $y'$ is predicted label

$l^H(y')$ = Column vector of size $n$ where each row $i$ is $l^H(i, y')$

$l^H_A(y, y')$ = Asymmetric loss where $y$ is the true label and $y'$ is predicted label

$l^H_A(y')$ = Column vector of size $n$ where each row $i$ is $l^H_A(i, y')$

$p$ = Column vector of size $n$ where each row is $p_i$

\begin{tcolorbox}
\textbf{Note:} If true label is $y$, loss incurred by predicting $y'$ is:

\begin{align*}
l^H(y, y') = \text{Shortest path length in H between y and y'}.\\
l^H_A(y, y') = \text{Shortest path length in H between y and y'.}
\end{align*}

For asymmetric tree, there is a pair of edges between any adjacent nodes, whose weights may necessarily not be equal. For example, let $y^1$, $y^2$ be adjacent nodes in hierarchy $H$, edge ($y^1$, $y^2$) points towards $y^2$ and edge ($y^2$, $y^1$) points towards $y^1$ and $l^H_A(y^1, y^2)$ $\neq$ $l^H_A(y^2, y^1)$. 
\end{tcolorbox}

\section{Bayes Optimal Classifier for the symmetric Tree-Distance loss}

\begin{theorem}
Let $H = ([n], E, W)$ and let $l^{H} : [n] \times [n] \rightarrow \mathbb{R_+}$ be the tree-distance loss for the tree $H$. For $x$ $\in$ $\chi$, let $p(x)$ $\in$ $\Delta_n$ be the conditional probability of the label given the instance $x$. Then there exists a $g^* : \chi \rightarrow [n]$ such that for all $x \in \chi$ the following holds:
\begin{enumerate}[label=(\alph*)]
    \item{$S_{g^*(x)}(p(x)) \geq \frac{1}{2}$}
    \item{$S_y(p(x)) \leq \frac{1}{2}, \forall y \in C(g^*(x)).$}
\end{enumerate}
And, $g^*$ is a Bayes optimal classifier for the symmetric tree distance loss ( \citet{ramaswamy2015convex})
\end{theorem}

The proof of sufficiency of Theorem 4.1 is given by \citet{ramaswamy2015convex}. We prove necessity of above theorem for bayes optimal classification in the following section.\\

\textbf{Note:} Symmetric Tree-distance loss $l^H$ and Asymmetric loss $l^H_A$ follows triangular inequality:
\begin{align*}
l^H(a, b) + l^H(b, c) \geq l^H(a, c) \\
l^H_A(a, b) + l^H_A(b, c) \geq l^H_A(a, c)
\end{align*}

\subsection{Necessity of Theorem 4.1 for Bayes Optimal Classification for the symmetric Tree-Distance loss}

\begin{theorem}
If $\exists$ a node $y \in [n]$, $S_y(p) < \frac{1}{2}$, then $y$ cannot be Bayes Optimal Classification.
\end{theorem}

\begin{proof}
Let $y'$ be a node where $S_{y'}(p) < \frac{1}{2}$ and $y^* = P(y')$.\\
Bayes loss for predicting $y$ = $\langle p, l^H(y) \rangle$.\\
Consider $\langle p, l^H(y') \rangle - \langle p, l^H(y^*) \rangle$:
\begin{equation*}
    \begin{split}
    \langle p, l^H(y') \rangle - \langle p, l^H(y^*) \rangle & = \sum\limits_{y \in D(y')} p_y(l^H(y, y')-l^H(y, y^*)) + \sum\limits_{y \in [n]\setminus D(y')} p_y(l^H(y, y')-l^H(y, y^*))\\
    & = \sum\limits_{y \in D(y')} p_y(-l^H(y', y^*)) + \sum\limits_{y \in [n]\setminus D(y')} p_y(l^H(y^*, y'))\\
    & = l^H(y', y^*)(-S_{y'}(p) + 1 - S_{y'}(p))\\
    & = l^H(y', y^*)(1 - 2S_{y'}(p))\\
    & > 0
    \end{split}
\end{equation*}
Thus, predicting $y^*$ is more optimal than $y'$. Hence, $y'$ is not a Bayes Optimal Classification.
\end{proof}

\begin{theorem}
If $\exists$ a node $y' \in [n]$, $S_{y'}(p) > \frac{1}{2}$ and $\exists y^* \in D(y')-\{y'\}, S_{y^*}(p) > \frac{1}{2}$, Then $y'$ cannot be Bayes optimal classification i.e. predicting $y^*$ is more Bayes optimal than $y'$.
\end{theorem}

\begin{proof}
Let $y^* \in D(y')-\{y'\}$, $S_{y^*}(p) > \frac{1}{2}$. So $y'$ is ancestor of $y^*$.\\
Consider $\langle p, l^H(y') \rangle - \langle p, l^H(y^*) \rangle$:
\begin{equation*}
    \begin{split}
    \langle p, l^H(y') \rangle - \langle p, l^H(y^*) \rangle & = \sum\limits_{y \in D(y^*)} p_y(l^H(y, y')-l^H(y, y^*)) + \sum\limits_{y \in [n]\setminus D(y^*)} p_y(l^H(y, y')-l^H(y, y^*))\\
    & = \sum\limits_{y \in D(y^*)} p_y(l^H(y^*, y')) + \sum\limits_{y \in [n]\setminus D(y^*)} p_y(l^H(y, y')-l^H(y, y^*))\\
    & \geq \sum\limits_{y \in D(y^*)} p_y(l^H(y^*, y')) + \sum\limits_{y \in [n]\setminus D(y^*)} p_y(-l^H(y', y^*))\\
    & \geq l^H(y', y^*)(S_{y^*}(p) - 1 + S_{y^*}(p))\\
    & \geq l^H(y', y^*)(2S_{y^*}(p) - 1)\\
    & > 0
    \end{split}
\end{equation*}
Thus, predicting $y^*$ is more optimal than $y'$. Hence, $y'$ is not a Bayes Optimal Classification.
\end{proof}
From Theorem 4.2 and Theorem 4.3, the conditions mentioned in Theorem 4.1 are necessary for Bayes Optimal Classification over symmetric Tree-distance loss.

\section{$\mathcal{O}(nk\log{}n)$ algorithm for finding Bayes optimal classification for k-ary tree as hierarchy}

The naive algorithm computes risk of predicting each node in $\mathcal{O}(n\log{}n)$ time. Thus, time complexity $\mathcal{O}(n^2\log{}n)$. In this section, we present the algorithm which computes risk of predicting a node $y^p$ in $\mathcal{O}(k\log{}n)$ time resulting in $\mathcal{O}(nk\log{}n)$ time complexity.

Let $K(y') = \sum\limits_{y \in D(y')} p_y(l^H(y, y'))$.\\
If $y'$ is a leaf node, then $K(y') = 0$. \\
Let $y^*$ be a node and $y_1$ and $y_2$ be its children and assume that $K(y_1)$ and $K(y_2)$ are computed. Then,\\

\begin{equation*}
    \begin{split}
        \sum\limits_{y \in D(y^*)} p_y(l^H(y, y^*)) &= \sum\limits_{y \in D(y_1)} p_y(l^H(y, y^*)) + \sum\limits_{y \in D(y_2)} p_y(l^H(y, y^*)) \\
        & = \sum\limits_{y \in D(y_1)} p_y(l^H(y, y_1) + l^H(y_1, y^*)) + \sum\limits_{y \in D(y_2)} p_y(l^H(y, y_2) + l^H(y_2, y^*)) \\
        & = \sum\limits_{y \in D(y_1)} p_yl^H(y, y_1) + \sum\limits_{y \in D(y_1)}p_yl^H(y_1, y^*)\\   
        & + \sum\limits_{y \in D(y_2)} p_yl^H(y, y_2) + \sum\limits_{y \in D(y_2)} p_yl^H(y_2, y^*) \\
        & = K(y_1) + l^H (y_1, y^*)S_{y_1}(p) + K(y_2) + l^H(y_2, y^*)S_{y_2}(p) \\
    \end{split}
\end{equation*}

Since $y_1$ and $y_2$ are children of $y^*$, $l^H (y_1, y^*)$, $l^H (y_2, y^*)$ are edge lengths and thus $l^H (y_1, y^*)S_{y_1}(p)$, $l^H (y_2, y^*)S_{y_2}(p)$ can be computed in $\mathcal{O}(1)$ time. \\
\\
If $K(y_1)$ and $K(y_2)$ are precomputed and since $l^H (y_1, y^*)S_{y_1}(p)$, $l^H (y_2, y^*)S_{y_2}(p)$ can be computed in $\mathcal{O}(1)$ time, $K(y^*)$ can be computed in $\mathcal{O}(1)$ time.

Thus, in bottom-up fashion, $K(y)$ for all nodes $y$ can be computed in $\mathcal{O}(n)$ time.

Let the node for which we are calculating risk be $y^p$.\\
\begin{align*}
Risk = \sum\limits_{y} p_y(l^H(y, y^p)) \\
\end{align*}
\begin{figure}[h]
  \begin{center}
  \includegraphics[scale=0.5]{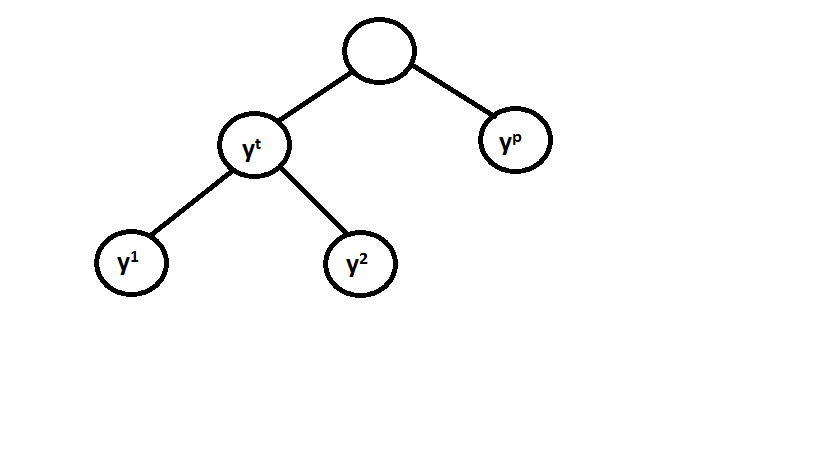}
  \end{center}
  \caption{Case-1}
  \label{fig:boat1}
\end{figure}

\textbf{Case 1: $y^t \notin D(y^p) \wedge y^p \notin D(y^t)$}\\
Let's compute risk of predicting $y^p$ with respect to descendants of $y^t$. \\ Consider $\sum\limits_{y \in D(y^t)} p_y(l^H(y, y^p))$:

\begin{equation*}
    \begin{split}
        \sum\limits_{y \in D(y^t)} p_y(l^H(y, y^p)) & = \sum\limits_{y \in D(y^t)} p_y(l^H(y, y^t)) + \sum\limits_{y \in D(y^t)} p_y(l^H(y^t, y^p))\\
        & = K(y^t) + \sum\limits_{y \in D(y^t)} p_y(l^H(y^t, y^p))\\
        & = K(y^t) + S_{y^t}(p)(l^H(y^t, y^p))\\
    \end{split}
\end{equation*}

$l^H(y^t, y^p)$ can be computed by traversing from $y^t$ to $y^p$ in $\mathcal{O}(\log{}n)$ time. Since $K(y^t)$ is precomputed, this computation takes $\mathcal{O}(\log{}n)$ time.\\ \\
\textbf{Case 2: $y^t \in D(y^p)$}\\
\begin{figure}[h]
  \begin{center}
  \includegraphics[scale=0.5]{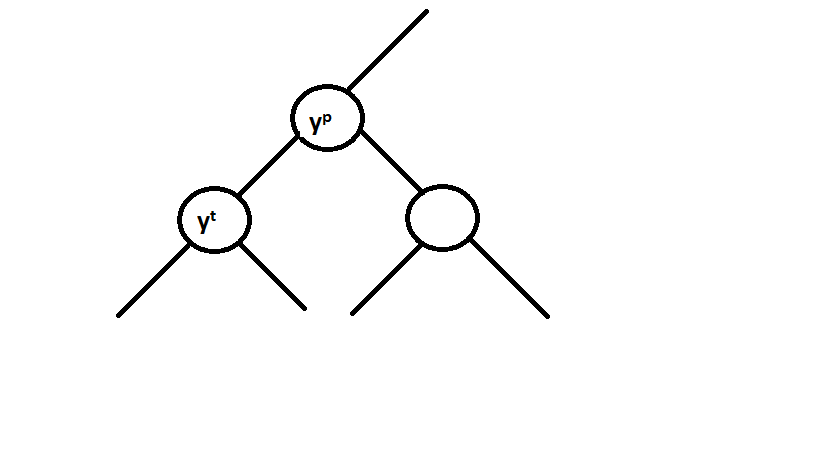}
  \end{center}
  \caption{Case-2}
  \label{fig:boat1}
\end{figure}
\begin{equation*}
    \begin{split}
        \sum\limits_{y \in D(y^p)} p_y(l^H(y, y^p)) & = K(y^p) \\
    \end{split}
\end{equation*}
Thus,  $\mathcal{O}(1)$ time.\\ \\
\newpage
\textbf{Case 3: $y^t \in U(y^p)$}\\
\begin{figure}[h]
  \begin{center}
  \includegraphics[scale=0.5]{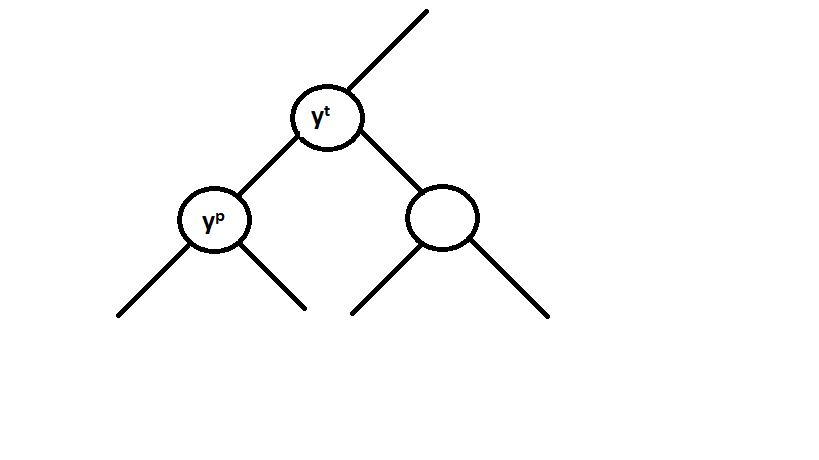}
  \end{center}
  \caption{Case-3}
  \label{fig:boat1}
\end{figure}
We have to compute $\sum\limits_{y \in U(y^p)} p_y(l^H(y, y^p))$. This can be computed by traversing from root to $y^p$, thus $\mathcal{O}(\log{}n)$ time.

The algorithm to compute Bayes optimal classification in $\mathcal{O}(nk\log{}n)$ time is described in the following section.\\ 
\newpage
\subsection{Algorithm}
\begin{algorithm}
\SetAlgoLined
\KwData{$H = ([n], E, W)$ and for each node $y$: $p_y$, $K(y)$, $S_y(p)$}
\KwResult{Bayes optimal classification $(Bopt)$}
 $min\_Risk = MAX$\;
 \For{each node $y^p$}{
  $Risk = 0$\;
  $Risk$ $+=$ $K(y^p)$\tcc*{Case-2}
  Traverse from $root$ to $y^p$ and by using running sum, compute $A = \sum\limits_{y \in U(y^p)} p_y(l^H(y, y^p))$ \;
  $Risk$ $+= A$ \tcc*{Case-3}
  \For{each node $y_1$ in the path from $root$ to $y^p$}{
    \For{each node $y_2$ $\in$ $sibling(y_1)$}{
        $Risk$ $+= K(y_2) + S_{y_2}(p)(l^H(y_2, y^p))$ \tcc*{$y_2$ $\in$ node in Case-1}
    }
  }
  \If{$Risk < min\_Risk$ }{
   $min\_Risk$ $=$ $Risk$ \; 
   $Bopt$ $=$ $y^p$ \;
  }
 }
 return $Bopt$ \;
 \caption{Bayes optimal classification}
\end{algorithm}

\subsection{Time complexity analysis}
Innermost for-loop loops $k$ times and each loop takes $\mathcal{O}(1)$ time to compute $l^H(y_2, y^p)$ because $P(y_2)$, parent of every such node $y_2$ is an ancestor of $y^p$ and since $l^H(P(y_2), y^p)$ is already computed in Case-3, it takes $\mathcal{O}(1)$ time and the for-loop surrounding it, loops $\mathcal{O}(\log{}n)$ times. Thus, $\mathcal{O}(k\log{}n)$. Since outer for-loop loops $n$ times, time complexity for k-ary tree is $\mathcal{O}(nk\log{}n)$, for binary tree, it is $\mathcal{O}(n\log{}n)$. 

\newpage
\section{Bayes optimal classification over asymmetric loss for a given hierarchy.}

\subsection{Assumptions over asymmetric loss}

We assume following on asymmetric loss throughout our work:
\begin{itemize}
    \item{$\forall y$ $\in$ $[n]$, $\forall y^1 \in C(y)$, $\frac{l^H_A(y, y^1)}{l^H_A(y, y^1)+l^H_A(y^1, y)}$,  increases down the tree.}
    \item{$\forall y \in C(r)$, $\frac{l^H_A(y, r)}{l^H_A(r, y)} \leq 1$}
\end{itemize}
\subsection{Sufficiency of Bayes optimal classification over asymmetric loss for a given hierarchy}
\begin{theorem}
Let $H = ([n], E, W)$ and let $l^{H}_A : [n] \times [n] \rightarrow \mathbb{R_+}$ be the asymmetric loss for the tree $H$. For $x$ $\in$ $\chi$, let $p(x)$ $\in$ $\Delta_n$ be the conditional probability of the label given the instance $x$. Assuming $\forall y \in [n]$, $\forall$ $y^1 \in C(y)$, $\frac{l^H_A(y, y^1)}{l^H_A(y, y^1)+l^H_A(y^1, y)}$ increases down the tree, $S_y(p)-\frac{l^H_A(P(y), y)}{l^H_A(P(y), y)+l^H_A(y, P(y))}$ decreases down the tree.
\end{theorem}

\begin{proof}
$S_y(p)$ decreases down the tree and from the assumption $\frac{l^H_A(P(y), y)}{l^H_A(P(y), y)+l^H_A(y, P(y))}$ increases down the tree. Hence, the difference between them decreases down the tree.
\end{proof}
\begin{theorem}
Let $H = ([n], E, W)$ and let $l^{H}_A : [n] \times [n] \rightarrow \mathbb{R_+}$ be the asymmetric loss for the tree $H$. For $x$ $\in$ $\chi$, let $p(x)$ $\in$ $\Delta_n$ be the conditional probability of the label given the instance $x$. Assuming $\forall$ $y \in [n]$, $\forall y^1 \in C(y)$, $\frac{l^H_A(y, y^1)}{l^H_A(y, y^1)+l^H_A(y^1, y)}$ increases down the tree and $\forall y \in C(r)$, $\frac{l^H_A(y, r)}{l^H_A(r, y)} \leq 1$, For any $y^1$, $y^2 \in [n]$, $\frac{l^H_A(P(y^1), y^1)}{l^H_A(P(y^1), y^1)+l^H_A(y^1, P(y^1))} + \frac{l^H_A(P(y^2), y^2)}{l^H_A(P(y^2), y^2)+l^H_A(y^2, P(y^2))} \geq 1$.
\end{theorem}

\begin{proof}
$\forall y \in C(r)$, $\frac{l^H_A(y, r)}{l^H_A(r, y)} \leq 1$ $\rightarrow$ $\frac{l^H_A(r, y)}{l^H_A(r, y)+l^H_A(y, r)} \geq 0.5$.\\
Since $\forall$ $y^1 \in C(y)$ $\frac{l^H_A(y, y^1)}{l^H_A(y, y^1)+l^H_A(y^1, y)}$ increases down the tree, for any $y \in [n]$,$\frac{l^H_A(P(y), y)}{l^H_A(P(y), y)+l^H_A(y, P(y))}\\ \geq 0.5$. Hence, sum of any two will be greater than equal to one. 
\end{proof}
\newpage
\begin{theorem}
Let $H = ([n], E, W)$ and let $l^{H}_A : [n] \times [n] \rightarrow \mathbb{R_+}$ be the asymmetric loss for the tree $H$. For $x$ $\in$ $\chi$, let $p(x)$ $\in$ $\Delta_n$ be the conditional probability of the label given the instance $x$. Assuming $\forall y \in [n]$, $\forall y^1 \in C(y)$, $\frac{l^H_A(y, y^1)}{l^H_A(y, y^1)+l^H_A(y^1, y)}$ increases down the tree and $\forall y \in C(r)$, $\frac{l^H_A(y, r)}{l^H_A(r, y)} \leq 1$, there exists a $g^* : \chi \rightarrow [n]$ such that for all $x \in \chi$ the following holds:
\begin{enumerate}[label=(\alph*)]
    \item{$S_{g^*(x)}(p(x)) \geq \frac{l^H_A(P(g^*(x)), g^*(x))}{l^H_A(P(g^*(x)), g^*(x))+l^H_A(g^*(x), P(g^*(x)))}$}
    \item{$S_y(p(x)) \leq \frac{l^H_A(g*(x), y)}{l^H_A(g^*(x), y)+l^H_A(y, g*(x))}$, $\forall$ $y$ $\in$ $C(g*(x))$}
\end{enumerate}
And, $g^*$ is a Bayes optimal classifier for the asymmetric loss.
\end{theorem}
\begin{proof}
Let $y^* \in [n]$ which follows above conditions i.e. 
\begin{align*}
    & S_{y^*}(p(x)) \geq \frac{l^H_A(P(y^*), y^*)}{l^H_A(P(y^*), y^*)+l^H_A(y^*, P(y^*))} \tag{1}\\
    & \forall y \in C(y^*), S_y(p(x)) \leq \frac{l^H_A(y^*, y)}{l^H_A(y^*, y)+l^H_A(y, y^*)} \tag{2}
\end{align*}
Now we show that $y^*$ minimizes $\langle p, l^H_A(y) \rangle$ over $y \in [n]$.\\ \\
Let $y'$ $\in$ $\argmin_t \langle p, l^H_A(t) \rangle$. If $y' = y^*$ we are done, hence assume $y' \neq y^*$.
\\ \\
\textbf{Case 1:} $y' \in D(y^*)\setminus C(y^*)$\\
Let $\widehat{y}$ be the child of $y^*$ that is the ancestor of $y'$. Hence, $S_{\widehat{y}}(p) \leq \frac{l^H_A(y^*, \widehat{y})}{l^H_A(y^*, \widehat{y})+l^H_A(\widehat{y}, y^*)}$.
\begin{equation*}
    \begin{split}
    \langle p, l^H_A(y') \rangle - \langle p, l^H_A(y^*) \rangle & = \sum\limits_{Path(y^*, y') \setminus y^*} l^H_A(P(y),y)-S_y(p)[l^H_A(P(y),y)+l^H_A(y,P(y)]\\
    & \geq 0
    \end{split}
\end{equation*}
From Theorem 6.1, the difference inside the above summation increases down the tree. As our summation is traversing the tree top-down, if first difference computed is greater than equal to zero, then all the rest will be greater than equal to zero. Since the first difference is computed at $y = \widehat{y}$ and since $S_{\widehat{y}}(p) \leq \frac{l^H_A(y*, \widehat{y})}{l^H_A(y*,\widehat{y})+l^H_A(\widehat{y}, y*)}$, the difference is greater than equal to zero. Hence, whole summation is greater than equal to zero. 
\newpage
\textbf{Case 2:} $y' \in C(y^*)$ \\
Since $y'$ is a child of $y^*$, $S_{y'}(p) \leq \frac{l^H_A(y^*, y')}{l^H_A(y^*, y')+l^H_A(y', y^*)}$.

\begin{equation*}
    \begin{split}
    \langle p, l^H_A(y') \rangle - \langle p, l^H_A(y^*) \rangle & = \sum\limits_{y \in D(y')} p_y(l^H_A(y, y')-l^H_A(y, y^*)) \\
    & + \sum\limits_{y \in [n]\setminus D(y')} p_y(l^H_A(y, y')-l^H_A(y, y^*))\\
    & = \sum\limits_{y \in D(y')} p_y(-l^H_A(y', y^*)) + \sum\limits_{y \in [n]\setminus D(y')} p_y(l^H_A(y^*, y'))\\
    & = -S_{y'}(p)l^H_A(y', y^*) + l^H_A(y^*, y')[1 - S_{y'}(p)]\\
    & = l^H_A(y^*, y') - S_{y'}(p)[l^H_A(y', y^*) + l^H_A(y^*, y')]\\
    & \geq 0
    \end{split}
\end{equation*}
\\ \\
\textbf{Case 3:} $y^* \in D(y')$\\
\begin{equation*}
    \begin{split}
    \langle p, l^H_A(y') \rangle - \langle p, l^H_A(y^*) \rangle & = \sum\limits_{Path(y', y^*) \setminus y'} S_y(p)[l^H_A(P(y),y)+l^H_A(y,P(y)]-l^H_A(P(y),y)\\
    & \geq 0
    \end{split}
\end{equation*}  
From Theorem 6.1, the difference inside the above summation decreases down the tree. As our summation is traversing the tree top-down, if the last difference computed is greater than equal to zero, then all the rest will be greater than equal to zero. Since the last difference is computed at $y = y^*$ and since $S_{y^*}(p(x)) \geq \frac{l^H_A(P(y^*), y^*)}{l^H_A(P(y^*), y^*)+l^H_A(y^*, P(y^*))}$, the difference is greater than equal to zero. Hence, whole summation is greater than equal to zero.
\\ \\ \\
\textbf{Case 4:} $y' \notin D(y^*) \wedge y^* \notin D(y')$\\
Let $y^2$ be the least common ancestor of $y'$ and $y^*$, $\widehat{y}$ be the child of $y^2$ which is an ancestor of $y'$ and $y^3$ be the child of $y^2$ which is an ancestor of $y^*$.
\begin{equation*}
    \begin{split}
    \langle p, l^H_A(y') \rangle - \langle p, l^H_A(y^*) \rangle & = \sum\limits_{Path(y^2, y') \setminus y^2} l^H_A(P(y),y)-S_y(p)[l^H_A(P(y),y)+l^H_A(y,P(y)] +\\ & \sum\limits_{Path(y^2, y^*) \setminus y^2} S_y(p)[l^H_A(P(y),y)+l^H_A(y,P(y)]-l^H_A(P(y),y)\\ & \geq 0
    \end{split}
\end{equation*}
From Theorem 6.1, the difference inside first summation increases down the tree and the difference inside second summation decreases down the tree. As our summation is traversing the tree top-down, if the first difference computed in first summation is greater than equal to zero, then entire first summation will be greater than equal to zero and similarly, if the last difference computed in second summation is greater than equal to zero, then entire second summation will be greater than equal to zero. Since the first difference in first summation is computed at $\widehat{y}$ and the last difference in second summation is computed at $y^*$, the following have to hold true for the whole expression to be greater than equal to zero: \\
\begin{align*}
    & S_{\widehat{y}}(p) \leq \frac{l^H_A(y^2, \widehat{y})}{l^H_A(y^2, \widehat{y})+l^H_A(\widehat{y}, y^2)} \tag{3}\\
    & S_{y^*}(p) \geq \frac{l^H_A(P(y^*), y^*)}{l^H_A(P(y^*), y^*)+l^H_A(y^*, P(y^*))} \tag{4}
\end{align*}
Equation 4 holds true from definition of $y^*$. Observe that:
\begin{equation}\tag{5}
    \begin{split}
        S_{y^2}(p) & \leq 1\\
        & \leq \frac{l^H_A(y^2, \widehat{y})}{l^H_A(y^2, \widehat{y})+l^H_A(\widehat{y}, y^2)} + \frac{l^H_A(P(y^*), y^*)}{l^H_A(P(y^*), y^*)+l^H_A(y^*, P(y^*))} \\
    \end{split} 
\end{equation}
Consider $S_{\widehat{y}}(p)$:
\begin{equation}\tag{6}
    \begin{split}
        S_{\widehat{y}}(p) &= S_{y^2} - S_{y^3}\\
        & \leq S_{y^2} - S_{y^*}\\
        & \leq S_{y^2} - \frac{l^H_A(P(y^*), y^*)}{l^H_A(P(y^*), y^*)+l^H_A(y^*, P(y^*))}\\
        & \leq \frac{l^H_A(y^2, \widehat{y})}{l^H_A(y^2, \widehat{y})+l^H_A(\widehat{y}, y^2)}
    \end{split} 
\end{equation}
From equations (5) and (6), we can get the above result and hence, both conditions (3), (4) hold true and thus $\langle p, l^H_A(y') \rangle$ $-$ $\langle p, l^H_A(y^*) \rangle$ $\geq$ 0.\\

Putting all four cases together we have:
\begin{align*}
\langle p, l^H_A(y^*) \rangle \leq \langle p, l^H_A(y') \rangle = \min\limits_{y \in [n]} \langle p, l^H_A(y) \rangle. 
\end{align*}
Hence, proved.
\end{proof}
\section{Necessity of Theorem 6.3 for Bayes Optimal Classification over asymmetric Tree-Distance loss}
\begin{theorem}
For a node $y \in [n]$ if $S_y(p) < \frac{l^H_A(P(y),y)}{l^H_A(P(y),y)+l^H_A(y,P(y)}$, then $y$ cannot be Bayes Optimal Classification.
\end{theorem}
\begin{proof}
Let $y'$ be a node where $S_{y'}(p) < \frac{l^H_A(P(y'),y')}{l^H_A(P(y'),y')+l^H_A(y',P(y')}$ and $y* = P(y')$.\\
Bayes loss for predicting $y$ = $\langle p, l^H_A(y) \rangle$. This is similar to case-2 in Theorem-6.3.\\
Consider $\langle p, l^H_A(y') \rangle - \langle p, l^H_A(y^*) \rangle$:
\begin{equation*}
    \begin{split}
    \langle p, l^H_A(y') \rangle - \langle p, l^H_A(y^*) \rangle & = l^H_A(y^*, y') - S_{y'}(p)[l^H_A(y', y^*) + l^H_A(y^*, y')]\\
    & > 0
    \end{split}
\end{equation*}
Thus, predicting $y^*$ is more optimal than $y'$. Hence, $y'$ is not a Bayes Optimal Classification.
\end{proof}

\begin{theorem}
For a node $y' \in [n]$ if $S_{y'}(p) > \frac{l^H_A(P(y'),y')}{l^H_A(P(y'),y')+l^H_A(y',P(y'))}$ and \\ $\exists y^* \in D(y')-\{y'\},$ $S_{y^*}(p) > \frac{l^H_A(P(y^*),y^*)}{l^H_A(P(y^*),y^*)+l^H_A(y^*,P(y^*))}$. Then $y'$ cannot be Bayes optimal classification i.e. predicting $y^*$ is more Bayes optimal than $y'$.
\end{theorem}
\begin{proof}
This is similar to case-3 in Theorem-6.3.\\ Consider $\langle p, l^H_A(y') \rangle - \langle p, l^H_A(y^*) \rangle$:
\begin{equation*}
    \begin{split}
    \langle p, l^H_A(y') \rangle - \langle p, l^H_A(y^*) \rangle & = \sum\limits_{Path(y', y^*) \setminus y'} S_y(p)[l^H_A(P(y),y)+l^H_A(y,P(y)]-l^H_A(P(y),y)\\
    & > 0
    \end{split}
\end{equation*}
Thus, predicting $y^*$ is more optimal than $y'$. Hence, $y'$ is not a Bayes Optimal Classification.
\end{proof}
From Theorem 7.1 and Theorem 7.2, the conditions mentioned in Theorem 6.3 are necessary for Bayes Optimal Classification over asymmetric Tree-distance loss.
\newpage
\section{Algorithm to find Bayes optimal classification over asymmetric loss under assumptions}
Since $\forall y \in C(r)$, $\frac{l^H_A(y, r)}{l^H_A(r, y)} \leq 1$ $\rightarrow$ $\frac{l^H_A(r, y)}{l^H_A(r, y)+l^H_A(y,r)} \geq 0.5$ and $\forall y \in [n]$, $\forall y^1 \in C(y)$, $\frac{l^H_A(y, y^1)}{l^H_A(y, y^1)+l^H_A(y^1, y)}$ increases down the tree, we can design $\mathcal{O}(\log{}n)$ algorithm for a binary tree and the following algorithm is for a binary tree as hierarchy and it can be easily extended for a k-ary tree.
\subsection{Algorithm}
\begin{algorithm}
\SetAlgoLined
\KwData{$H = ([n], E, W)$, $l^{H}_A : [n] \times [n] \rightarrow \mathbb{R_+}$ and for each node $y$: $p_y$, $S_y(p)$}
\KwResult{Bayes optimal classification $(Bopt)$}
 $y$ = root \;
 $value$ = 1 \;
 \While{$value$ $\neq 0$}{
    \If{$isLeaf(y)$}{
        break\;
    }
    
    $y^1$, $y^2$ = C(y) \;
    \If{$S_{y^1}(p) \geq \frac{l^H_A(y,y^1)}{l^H_A(y,y^1)+l^H_A(y^1,y)}$}{
        $y$ = $y^1$ \;
        $value$ = 1 \;
    }
    \ElseIf{$S_{y^2}(p) \geq \frac{l^H_A(y,y^2)}{l^H_A(y,y^2)+l^H_A(y^2,y)}$}
    {
        $y$ = $y^2$ \;
        $value$ = 1 \;
    }
    \Else{
    $y$ = $y^2$ \;
    $value$=0 \;}
  }
  $Bopt$ = ($value$) $?$ $y$ : $P(y)$ \;
  return $Bopt$ \;
 \caption{Bayes optimal classification}
 \label{alg:2}
\end{algorithm}

In words, we start at the root node and keep on moving to the child of current node that satisfies the conditions mentioned in Theorem-7.3 and terminate when we reach a leaf node, or a node where all of its children fail the conditions.
\subsection{Time complexity analysis}
This is a tree traversal from root node to the bayes optimal classification where in the worst case, it visits all children of one node per level. So, for a k-ary tree, time complexity is $\mathcal{O}(k\log{}n)$ and for a binary tree, time complexity is $\mathcal{O}(\log{}n)$.

\section{Results}
The experiments are run on CLEF dataset (\citet{dimitrovski2011hierarchical}). CLEF dataset consists of Medical X-ray images organized according to a hierarchy. We use tree-distance loss in (\citet{sun2001hierarchical}) as evaluation metric. We vary training methods and loss metrics. Symmetric method is the bayes optimal classifier proposed by (\citet{ramaswamy2015convex}) and  asymmetric method is the algorithm-\ref{alg:2}. We convert asymmetric tree to symmetric tree by replacing the up and down edges with a single undirected edge whose weight is the average of the up and down weights. The results are shown in table-1.

\begin{table}[h!]
\begin{center}
\label{table:result}
\begin{tabular}{ |l|c|c| } 
 \hline
 {\bf Training method} & {\bf Loss metric} & {\bf Tree distance loss}\\
 \hline
  Asymmetric method & Asymmetric loss & $0.72$ \\ 
  Asymmetric method & Symmetric loss & $0.82$ \\ 
  Symmetric method & Asymmetric loss & $0.75$ \\ 
  Symmetric method & Symmetric loss & $0.80$ \\ 
 \hline
\end{tabular}
\end{center}
\caption{Empirical Results}
\end{table}

\section{Conclusion}
In this project, we propose $\mathcal{O}(nk\log{}n)$ algorithm for finding bayes optimal classification over symmetric/asymmetric loss metric for k-ary tree as hierarchy. We propose $\mathcal{O}(k\log{}n)$ algorithm for finding bayes optimal classification over asymmetric loss under reasonable assumptions for a k-ary tree as hierarchy. From experiments, we conclude that given an asymmetric tree, one can't achieve good performance by converting asymmetric tree to symmetric tree and apply symmetric method. So, if we have an asymmetric tree then we have to use asymmetric method for achieving good performance.

\section{Future Improvements}
We can do following improvements and we are currently working on some of them:
\begin{itemize}
    \item We can include reject option i.e. classifier abstains from predicting if it is not confident enough.
    \item We can improve it for an interactive hierarchical classification which can be a huge boost in e-commerce industry.
    \item We can design a scalable algorithm with surrogates which uses our algorithm as basis.
\end{itemize}

\bibliographystyle{plainnat}
\bibliography{main}

\end{document}